\documentclass{article}

 \PassOptionsToPackage{numbers, compress}{natbib}

\usepackage{bm, amsmath, amsthm, amssymb, url}
\usepackage[pdftex]{graphicx}
\usepackage{ascmac}

\usepackage[preprint]{nips_2018}

\usepackage[utf8]{inputenc} 
\usepackage[T1]{fontenc}    
\usepackage{hyperref}       
\usepackage{url}            
\usepackage{booktabs}       
\usepackage{amsfonts}       
\usepackage{nicefrac}       
\usepackage{microtype}      

\newtheorem{Lemma}{Lemma}
\newtheorem{Theorem}{Theorem}

\title{Estimation of Non-Normalized Mixture Models and \\Clustering Using Deep Representation}

\author{
  Takeru Matsuda \\
  Department of Mathematical Informatics\\
  The University of Tokyo, Japan \\
  \texttt{matsuda@mist.i.u-tokyo.ac.jp} \\
   \And
  Aapo Hyv\"arinen \\
  Gatsby Computational Neuroscience Unit\\
  University College London, UK \\
  Department of Computer Science and HIIT\\
  University of Helsinki, Finland \\
  \texttt{a.hyvarinen@ucl.ac.uk} \\
}

\begin{document}

\maketitle

\begin{abstract}
We develop a general method for estimating a finite mixture of non-normalized models.
Here, a non-normalized model is defined to be a parametric distribution with an intractable normalization constant.
Existing methods for estimating non-normalized models without computing the normalization constant are not applicable to mixture models 
because they contain more than one intractable normalization constant.
The proposed method is derived by extending noise contrastive estimation (NCE), 
which estimates non-normalized models by discriminating between the observed data and some artificially generated noise.
We also propose an extension of NCE with multiple noise distributions. 
Then, based on the observation that conventional classification learning with neural networks is implicitly assuming an exponential family as a generative model,
we introduce a method for clustering unlabeled data by estimating a finite mixture of distributions in an exponential family.
Estimation of this mixture model is attained by the proposed extensions of NCE where the training data of neural networks are used as noise.
Thus, the proposed method provides a probabilistically principled clustering method that is able to utilize a deep representation.
Application to image clustering using a deep neural network gives promising results.
\end{abstract}

\section{Introduction}
Our paper aims at combining two theoretical frameworks: non-normalized models, and mixture models; our motivating application is to learn clustering based on a representation learned by deep neural networks.

Many statistical models are given in the form of non-normalized densities with an intractable normalization constant; they are also called energy-based.
Since maximum likelihood estimation is computationally very intensive for these models,
several estimation methods have been developed which do not require the normalization constant (i.e.\ the partition function), or somehow estimate it as part of the estimation process.
These include pseudo-likelihood \cite{Besag}, contrastive divergence \cite{Hinton}, score matching \cite{SM}, and noise contrastive estimation \cite{Gutmann}. 

On the other hand, mixture models are a well-known general-purpose approach to unsupervised modelling of complex distributions, especially in the form of the Gaussian Mixture Model. In particular, estimation of a finite mixture model leads to a probabilistically principled clustering method. Compared to other clustering methods such as hierarchical clustering and K-means clustering,
such model-based methods naturally quantify the uncertainty of each membership.

An application where non-normalized models and mixture models naturally meet is learning a clustering based on features learned by a neural network.
Deep neural networks have been shown to learn useful representations from labeled data such as ImageNet, and such representations seem to be useful for analyzing other datasets, or for performing other tasks.
For example, many neural networks trained for ImageNet competition \cite{Krizhevsky,Szegedy} are publicly available 
and they work well as feature extractors in natural image processing.
Exploiting such a learned representation for other datasets or tasks is a fundamental case of transfer learning.
Although transfer learning is well established for supervised learning such as classification,
how to transfer the learned representation to unsupervised learning such as clustering is still unclear. A straightforward approach would be to use the neural network features in a mixture model, but what such a model means and how it can be estimated is unclear.


In this study, we develop a general method for estimating a finite mixture of non-normalized models. It is not known if any of the aforementioned methods is applicable in such a setting, since we have several normalization constants instead of a single one. The proposed method is expected to significantly increase the practicality of non-normalized mixture models, 
which have been scarse, presumably due to the lack of a practical estimation method. As an application of great practical interest, we apply the framework for transferring a deep image representation to clustering of unlabeled data.
Our approach provides a probabilistically principled solution for the clustering problem, building a probabilistic model that propagates back to the original data space. 

To accomplish our goal, first, we extend noise contrastive estimation (NCE) to a finite mixture of non-normalized models. We further propose an extension of NCE with multiple noise distributions. Then, we point out that classification learning with deep neural networks is implicitly assuming an exponential family as a generative model. Based on this observation, we propose a method for clustering unlabeled data by estimating a finite mixture of distributions in an exponential family that is derived from the deep representation. 
Estimation of this mixture model is attained by using the proposed extensions of NCE with a particular choice of noise.
Finally, we apply the proposed method to image clustering, with promising if preliminary results.

\section{Background: non-normalized models and noise contrastive estimation}
In this section, we briefly review the problem of non-normalized models, and its solution by noise contrastive estimation \cite{Gutmann}. 
Suppose we have $N$ samples $x_1,\cdots,x_N$ from a parametric distribution
\begin{equation}
	p(x \mid \theta) = \frac{1}{Z(\theta)} \widetilde{p}(x \mid \theta), \label{nnp}
\end{equation}
where $\theta$ is an unknown parameter and $Z(\theta)$ is the normalization constant.
For several statistical models such as Markov random fields \cite{Li} and energy-based overcomplete ICA models \cite{Teh}, 
only the non-normalized density $\widetilde{p}(x \mid \theta)$ is given and the calculation of $Z(\theta)$ is intractable.
Thus, several methods have been developed to estimate $\theta$ without explicitly computing $Z(\theta)$.
They include pseudo-likelihood \cite{Besag}, contrastive divergence \cite{Hinton}, score matching \cite{SM}, and noise contrastive estimation \cite{Gutmann}.

In noise contrastive estimation (NCE), the non-normalized model is rewritten as
\begin{equation}
	\log p(x \mid \theta,c) = \log \widetilde{p} (x \mid \theta) + c, \label{NCEparam}
\end{equation}
where the scalar $c=-\log Z(\theta)$ is also viewed as an unknown parameter and estimated from data.
In addition to data $x_1,\cdots,x_N$, we generate $M$ noise samples $y_1,\cdots,y_M$ from a noise distribution $n(y)$.
The noise distribution should be difficult to discriminate from the real data, while having a tractable probability density function.
For example, $n(y)$ can be set to the Gaussian distribution with the same mean and covariance with data.
Then, the estimate of $(\theta, c)$ is defined by learning to discriminate between the data and the noise as accurately as possible:
\begin{equation}
	(\hat{\theta}_{{\rm NCE}},\hat{c}_{{\rm NCE}}) = {\rm arg} \max_{\theta,c} \hat{J}_{{\rm NCE}} (\theta,c), \label{NCEdef}
\end{equation}
where 
\begin{align}
	\hat{J}_{{\rm NCE}} (\theta,c) =& \sum_{t=1}^N \log \frac{N p(x_t \mid \theta,c)}{N p(x_t \mid \theta,c)+M n(x_t)} + \sum_{t=1}^M \log \frac{M n(y_t)}{N p(y_t \mid \theta,c)+M n(y_t)}. \label{Jdef}
\end{align}
The objective function $\hat{J}_{{\rm NCE}}$ is the log-likelihood of the logistic regression classifier. 
NCE has consistency and asymptotic normality under mild regularity conditions \cite{Gutmann12}.
Note that NCE is somewhat similar in spirit to Generative Adversarial Networks \cite{GAN}, 
which aim to generate realistic data by training a generative network and a discriminative network simultaneously.

\section{Mixture of non-normalized models, and extensions of NCE}
In this section, we first define the problem of non-normalized mixture models. 
Then we develop a general method for estimating a finite mixture of non-normalized models by extending NCE 
and discuss its application to clustering. 
We also investigate an extension of NCE with multiple noise distributions, which will be used in Section 5.

\subsection{Definition of a finite mixture of non-normalized models}

Suppose we have $N$ samples $x_1,\cdots,x_N$ from a finite mixture distribution
\begin{align}
	p (x \mid \theta,\pi) &= \sum_{k=1}^K \pi_k \cdot p_k (x \mid \theta_k), \label{mix_model}
\end{align}
where 
\begin{align}
	p_k(x \mid \theta_k) = \frac{1}{Z(\theta_k)} \widetilde{p}_k(x \mid \theta_k).
\end{align}
Here, $\theta=(\theta_1,\cdots,\theta_K)$ and $\pi=(\pi_1,\cdots,\pi_K)$ are unknown parameters 
and the normalization constant $Z(\theta_k)$ of each component $p_k (x \mid \theta_k)$ is intractable.
Existing methods for estimating non-normalized models are not applicable to \eqref{mix_model} since it includes more than one intractable normalization constant.
Although \cite{Nair} extended the contrastive divergence method to estimate a finite mixture of restricted Boltzmann machines, that is only a special case.

\subsection{NCE for estimation of mixture of non-normalized distributions}

Here, we extend NCE to estimate \eqref{mix_model} in general.
First, we reparametrize \eqref{mix_model} as
\begin{align}
	p (x \mid \theta,c) = \sum_{k=1}^K p_k(x \mid \theta_k,c_k), \label{NCEparam2}
\end{align}
where $c=(c_1,\cdots,c_K)$ with $c_k = \log \pi_k-\log Z(\theta_k)$ and each $p_k(x \mid \theta_k,c_k)$ is defined as
\begin{equation}
	\log p_k (x \mid \theta_k,c_k) = \log \widetilde{p}_k (x \mid \theta_k) + c_k.
\end{equation}
When $K=1$, this reparametrization reduces to that used in the original NCE in \eqref{NCEparam}.
Similarly to the original NCE, we consider $c$ as an additional unknown parameter. 
Then, we generate $M$ noise samples $y_1,\cdots,y_M$ from a noise distribution $n(y)$ 
and estimate $(\theta, c)$ in the same way as the original NCE in \eqref{NCEdef} and \eqref{Jdef},
that is, we use the definition \eqref{NCEparam2} in the original NCE objective function \eqref{Jdef}.
This estimator has consistency under mild regularity conditions similar to the original NCE (see Supplementary Material).
Note that the additional parameter $c_k$ incorporates both the mixture weight $\pi_k$ and the normalization constant $Z(\theta_k)$ 
and so we cannot obtain an estimate of $\pi_k$ from the estimate of $c_k$,
although it is not a problem for clustering application as shown in the next paragraph.

The estimation result can be used for clustering of $x_1,\cdots,x_N$. 
Specifically, by introducing a hidden variable $z$ taking values in $\{ 1,\cdots,K \}$, 
the mixture model \eqref{mix_model} is rewritten in a hierarchical form:
\begin{equation}
	p(z=k \mid \pi) = \pi_k \quad (k=1,\cdots,K),
\end{equation}
\begin{equation}
	p (x \mid z=k; \theta) = p_k (x \mid \theta_k).
\end{equation}
Then, the posterior of $z$ given $x$ is
\begin{equation}
	p(z=k \mid x; \theta, \pi) = \frac{\pi_k p_k(x \mid \theta_k)}{\sum_{j=1}^K \pi_j p_j(x \mid \theta_j)} \quad (k=1,\cdots,K).
\end{equation}
Thus, based on the posterior $p(z_t=k \mid x_t; \hat{\theta}, \hat{\pi})$ for each $x_t$, clustering of $x_1,\cdots,x_N$ is obtained.


\subsection{NCE with multiple noise distributions}

While the preceding subsection solves the problem of estimating non-normalized mixture models, we next introduce another  extension of NCE which is useful in further developments below.
In the original NCE, we generate noise samples from one noise distribution and discriminate between data and noise.
In order that such discrimination learns deep structure in the data, 
it would intuitively seem important that the noise distribution is as close as possible to the real data distribution.
Thus, it would be more efficient to use several noise distributions, 
since different noise distributions would accelerate to learn different kinds of data structure.
Here, we introduce NCE with multiple noise distributions and discuss its equivalence to the original NCE with a mixture noise distribution.

Suppose we have $N$ samples $x_1,\cdots,x_N$ from a non-normalized distribution \eqref{nnp} or a finite mixture of non-normalized distributions \eqref{mix_model}.
We consider $L$ noise distributions $n_1(y),\cdots,n_L(y)$ and generate $M_l$ noise samples $y^{(l)}_1,\cdots,y^{(l)}_{M_l}$ from each $n_l(y)$.
Then, similarly to the original NCE and its extension in Section 3.2, 
an estimate of $(\theta, c)$ can be defined by discriminating between $L+1$ classes (data, noise 1, $\cdots$, noise $L$) as correctly as possible:
\begin{align}
	(\hat{\theta}_{{\rm MNCE}},\hat{c}_{{\rm MNCE}}) = {\rm arg} \max_{\theta,c} \hat{J}_{{\rm MNCE}} (\theta,c), \label{MNCE}
\end{align}
where
\begin{align}
	\hat{J}_{{\rm MNCE}} (\theta,c) =& \sum_{t=1}^N \log \frac{N p(x_t \mid \theta,c)}{N p(x_t \mid \theta,c)+M_1 n_1(x_t)+\cdots+M_L n_L(x_t)} \nonumber \\
	&+ \sum_{l=1}^L \sum_{t=1}^{M_l} \log \frac{M_l n_l(y^{(l)}_t)}{N p(y^{(l)}_t \mid \theta,c)+M_1 n_1(y^{(l)}_t)+\cdots+M_L n_L(y^{(l)}_t)}.
\end{align}

On the other hand, we can regard $y^{(1)}_1,\cdots,y^{(1)}_{M_1},\cdots,y^{(L)}_1,\cdots,y^{(L)}_{M_L}$ as samples from the mixture distribution
\begin{align}
	n(y) = \sum_{l=1}^L \frac{M_l}{M_1+\cdots+M_L} n_l(y), \label{noise_mixture}
\end{align}
and use the original NCE $(\hat{\theta}_{{\rm NCE}},\hat{c}_{{\rm NCE}})$ as \eqref{NCEdef}.

In fact, these two estimators coincide as follows:
\begin{Theorem}\label{th_NCE}
\begin{align}
	(\hat{\theta}_{{\rm MNCE}},\hat{c}_{{\rm MNCE}}) = (\hat{\theta}_{{\rm NCE}},\hat{c}_{{\rm NCE}}).
\end{align}
\end{Theorem}
The proof is given in Supplementary Material. From Theorem \ref{th_NCE}, NCE with multiple noise distributions has the same statistical properties with the original NCE.
We will present simulation results for a typical situation where using multiple noise distributions is beneficial in Section 6.

\section{Exponential family and classification with neural networks}
In this section, we lay the ground for applying the preceding developments to deep neural networks. 
We propose an interpretation where an exponential family is implicitly assumed as a generative model in classification learning with neural networks.
Such an interpretation was also pointed out by \cite{Dai,Xie}.
We consider image classification for convenience of terminology.

Let $x$ be image data and $z$ be its category. 
We assume that $z$ takes values in $\{ 1,\cdots,L \}$.
In classification with neural networks, the softmax function is commonly used in the output layer.
Namely, the probability output is computed by 
\begin{equation}
	p(z=l \mid x) = \frac{\exp(\sum_{i=1}^d w_{li} f_i(x))}{\sum_{j=1}^L \exp(\sum_{i=1}^d w_{ji} f_i(x))} \quad (l=1,\cdots,L), \label{deep_softmax}
\end{equation}
where $d$ is the number of units in the last hidden layer, 
$f_i(x)$ is the activation of the $i$-th unit in the last hidden layer when $x$ is input to this network, 
and $w_{ji}$ is the connection weight between the $i$-th unit in the last hidden layer and the $j$-th output unit.
Thus, neural networks learn to extract nonlinear features $f_1,\cdots,f_d$ that are useful for image classification.

From \eqref{deep_softmax}, we obtain
\begin{equation}
	\frac{p(z=l \mid x)}{p(z=1 \mid x)} = \exp \left( \sum_{i=1}^d (w_{li}-w_{1i}) f_i(x) \right) \quad (l=1,\cdots,L). \label{softmax_ratio}
\end{equation}
On the other hand, from Bayes' formula, we obtain
\begin{equation}
	\frac{p(z=l \mid x)}{p(z=1 \mid x)} = \frac{p(z=l) p(x \mid z=l)}{p(z=1) p(x \mid z=1)} \quad (l=1,\cdots,L), \label{bayes_formula}
\end{equation}
where the prior probability $p(z)$ is defined from the proportion of each category in the training data.
Therefore, \eqref{softmax_ratio} and \eqref{bayes_formula} lead to
\begin{equation}
	p (x \mid z=l) = p(x \mid z=1) \frac{p(z=1)}{p(z=l)} \exp \left( \sum_{i=1}^d (w_{li}-w_{1i}) f_i(x) \right). \label{pxzk}
\end{equation}

Now, consider an exponential family
\begin{equation}
	p (x \mid \theta) = h(x) \exp \left( \sum_{i=1}^d \theta_{i} f_i(x) - A(\theta) \right), \label{gen_model}
\end{equation}
where 
\begin{equation}
	h(x) = p(x \mid z=1) \exp \left( -\sum_{i=1}^d w_{1i} f_i(x) \right). \label{h_def}
\end{equation}
Then, from \eqref{pxzk}, the distribution of images in the $l$-th category belongs to this exponential family with $\theta_i=w_{li}$ and $A(\theta)=\log p(z=l)-\log p(z=1)$.
Thus, classification with neural networks \eqref{deep_softmax} implicitly assumes the exponential family \eqref{gen_model} as a generative model.

For image data, many pretrained networks are publicly available such as AlexNet \cite{Krizhevsky} and inception-v3 \cite{Szegedy}.
Although these networks were trained for ImageNet competition, they have learned a useful representation of general natural images.
Indeed, they work well empirically as feature extractors for other image data.
Therefore, the distributional assumption \eqref{gen_model} seems to be reasonable even for image categories outside of the ImageNet competition.


\section{Clustering with deep representation}
In this section, we combine the developments above to finally provide a method for transferring the representation of a deep neural network to clustering of unlabeled data, 
using the extensions of NCE proposed in Section 3 and the exponential family introduced in Section 4.
In the current state of research, it seems that the only way to employ the deep representation for clustering is to heuristically apply conventional clustering algorithms to the feature vectors.
Here, we provide a probabilistically principled clustering method that leverages the deep representation.
Again, for concreteness of exposition, we consider image clustering, although the method is quite general.

Suppose we have $N$ images $x_1,\cdots,x_N$ and a neural network previously trained (``pretrained") on some other image dataset (e.g., AlexNet, inception-v3).
We assume that $x_1,\cdots,x_N$ belongs to the same exponential family \eqref{gen_model} with the image data on which the network was pretrained, in other words, the difference is only in the last layer weights.
Then, the generative model of $x_1,\cdots,x_N$ is a finite mixture of distributions in the same exponential family:
\begin{align}
	p (x \mid \theta,\pi) &= \sum_{k=1}^K \pi_k \cdot h(x) \exp \left( \sum_{i=1}^d \theta_{ki} f_i(x) - A(\theta_k) \right), \label{EFMt}
\end{align}
where $K$ is the number of image categories in $x_1,\cdots,x_N$.
Note that $A(\theta_k)$ here are not known and intractable, although they were known for the categories used in training.
Like \eqref{NCEparam2}, we reparametrize \eqref{EFMt} as
\begin{align}
	p (x \mid \theta,c) = h(x) \sum_{k=1}^K \exp \left( \sum_{i=1}^d \theta_{ki} f_i(x) + c_k \right), \label{EFMt2}
\end{align}
where $c=(c_1,\cdots,c_K)$.
From \eqref{h_def}, the function $h$ is a function of the distribution of one image category $p(x \mid z=1)$ and so it is totally unknown.
Yet, clustering of $x_1,\cdots,x_N$ is possible if we can estimate $\theta$ and $c$, 
since the function $h$ cancels out in the posterior:
\begin{equation}
	p(z=k \mid x; \theta, c) = \frac{\exp \left( \sum_{i=1}^d \theta_{ki} f_i(x) + c_k \right)}{\sum_{j=1}^K \exp \left( \sum_{i=1}^d \theta_{ji} f_i(x) + c_j \right)} \quad (k=1,\cdots,K). \label{EFMpos}
\end{equation}

We use the NCE extensions in Section 3 to estimate $\theta$ and $c$ in \eqref{EFMt2}.
Here, we have to be careful in the choice of the noise distribution because of the unknown function $h$ in \eqref{EFMt2}.
If we generate noise samples artificially, $h$ remains in the objective function of NCE \eqref{Jdef} and so the optimization is impossible.
To get rid of $h$, we propose here to use the original training data of the pretrained network as noise samples.
Specifically, let $\widetilde{x}^{(1)}_1,\cdots,\widetilde{x}^{(1)}_{M_1},\cdots,\widetilde{x}^{(L)}_1,\cdots,\widetilde{x}^{(L)}_{M_L}$ be the training data of the pretrained network,
where $L$ is the number of categories and $M_l$ is the number of samples in the $l$-th category.
Then, the prior probability is $p(z=l) = M_l/M$ where $M=M_1+\cdots+M_L$.
Therefore, from \eqref{pxzk} and \eqref{h_def}, the distribution of images in the $l$-th pre-training category (here used as noise) is
\begin{equation}
	q_l (\widetilde{x}) = h(\widetilde{x}) \frac{M_1}{M_l} \exp \left( \sum_{i=1}^d w_{li} f_i(\widetilde{x}) \right) \quad (l=1,\cdots,L). \label{noise_dist}
\end{equation}
Thus, we regard $q_1,\cdots,q_L$ as noise distributions and the training data $\widetilde{x}^{(l)}_1,\cdots,\widetilde{x}^{(l)}_{M_l}$ as samples from $q_l$ for $l=1,\cdots,L$, respectively\footnote{In practice, using only categories relevant to the new data may suffice and it reduces computational cost.}.

In summary, the estimate of $(\theta,c)$ is given by
\begin{align}
	(\hat{\theta}_{{\rm MNCE}},\hat{c}_{{\rm MNCE}}) = {\rm arg} \max_{\theta,c} \hat{J}_{{\rm MNCE}} (\theta,c), \label{final_est}
\end{align}
where
\begin{align}
	&\hat{J}_{{\rm MNCE}} (\theta,c) = \sum_{t=1}^N \log \frac{N \sum_{k=1}^K \exp \left( \sum_{i=1}^d \theta_{ki} f_i(x_t) + c_k \right)}{N \sum_{k=1}^K \exp \left( \sum_{i=1}^d \theta_{ki} f_i(x_t) + c_k \right) + M_1 \sum_{l=1}^L \exp \left( \sum_{i=1}^d w_{l i} f_i(x_t) \right)} \nonumber \\
	&+ \sum_{l=1}^L \sum_{t=1}^{M_l} \log \frac{M_1 \exp \left( \sum_{i=1}^d w_{l i} f_i(\widetilde{x}^{(l)}_t) \right)}{N \sum_{k=1}^K \exp \left( \sum_{i=1}^d \theta_{ki} f_i(\widetilde{x}^{(l)}_t) + c_k \right) + M_1 \sum_{l=1}^L \exp \left( \sum_{i=1}^d w_{l i} f_i(\widetilde{x}^{(l)}_t) \right)}.
\end{align}
Note that $h$ cancels out in $\hat{J}_{{\rm MNCE}}$, and so the objective function only depends on quantities we can readily compute.
Using the estimate \eqref{final_est}, clustering of $x_1,\cdots,x_N$ is obtained by the posterior \eqref{EFMpos}.

\section{Simulation results}
In this section, we use simulations to further confirm the validity of the estimation of non-normalized mixture models by extensions of NCE proposed in Section 3.
As a special case of finite mixture models \eqref{NCEparam2}, we consider the one-dimensional Gaussian mixture distribution.
Namely,
\begin{equation}
	p(x \mid \theta,c) = \sum_{k=1}^K \exp (\theta_{k1} x^2 + \theta_{k2} x + c_k). \label{gm_def}
\end{equation}
where we pretend not to be able to compute the normalization constants for the purpose of this simulation.
We generated $N$ samples $x_1,\cdots,x_N$ from the two-component Gaussian mixture distribution $0.5 \cdot {\rm N} (0,1) + 0.5 \cdot {\rm N} (4,1)$.
The sample size $N$ was set to $2^9,2^{10},\cdots,2^{18}$ and the simulation was repeated 100 times for each sample size.

We consider two estimation methods, both of which are based on the proposed extensions of NCE. 
The first method is NCE with $M=N$ noise samples generated from the Gaussian distribution ${\rm N} (2,5)$, which has the same mean and variance with the true data-generating distribution $0.5 \cdot {\rm N} (0,1) + 0.5 \cdot {\rm N} (4,1)$.
The second method is NCE with $M_1=M_2=N/2$ noise samples generated from two Gaussian distributions ${\rm N} (0,1)$ and ${\rm N} (4,1)$.
We solved the optimization \eqref{NCEdef} in NCE by the nonlinear conjugate gradient method \cite{Rasmussen}.

Figure \ref{fig_mse} plots the median of the squared errors for $\theta$ and $c$ of each estimation method with respect to the sample size $N$.
Here, among the two estimated components that are non-normalized Gaussian distributions, we regarded the one with the smaller mean as the first component $p_1(x \mid \theta_1,c_1)$.
For $\theta$, we also plot the median of the squared error of the maximum likelihood estimator computed by the MATLAB function \textit{fitgmdist}.
The estimation errors converge to zero for both $\theta$ and $c$, which provides evidence for the consistency of NCE extensions.
Also, the estimation accuracy of the second method is slightly better than that of the first method, which is understood as follows.
From Theorem 1, the second method is equivalent to NCE with the noise distribution equal to the true data-generating distribution.
Therefore, noise in the second method is more difficult to discriminate from data than in the first method.

\begin{figure}
\begin{minipage}{0.45\textwidth}
(a)
	\begin{center}
	\includegraphics[width=6cm]{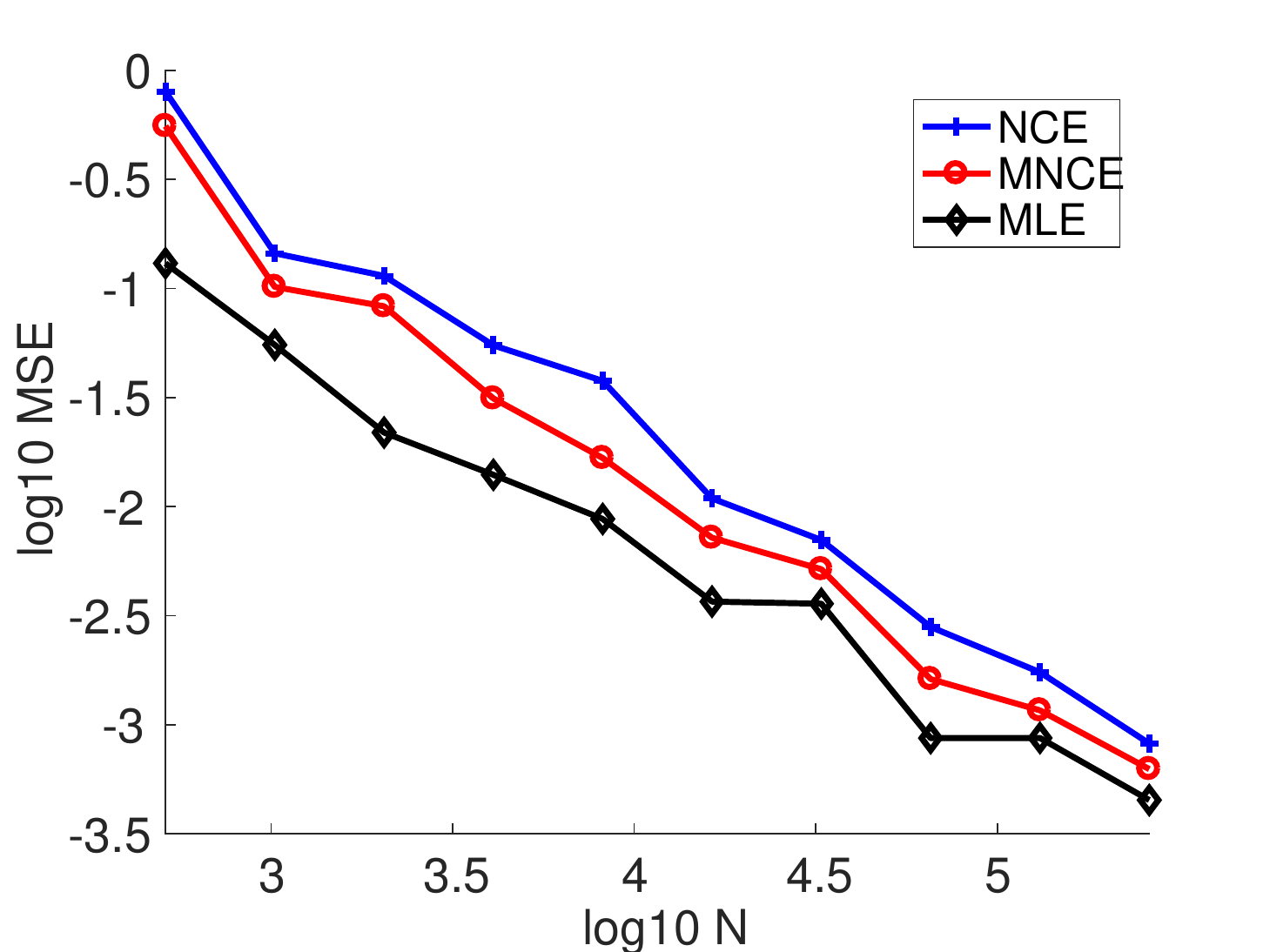}
	\end{center}
\end{minipage}
\begin{minipage}{0.45\textwidth}
(b)
	\begin{center}
	\includegraphics[width=6cm]{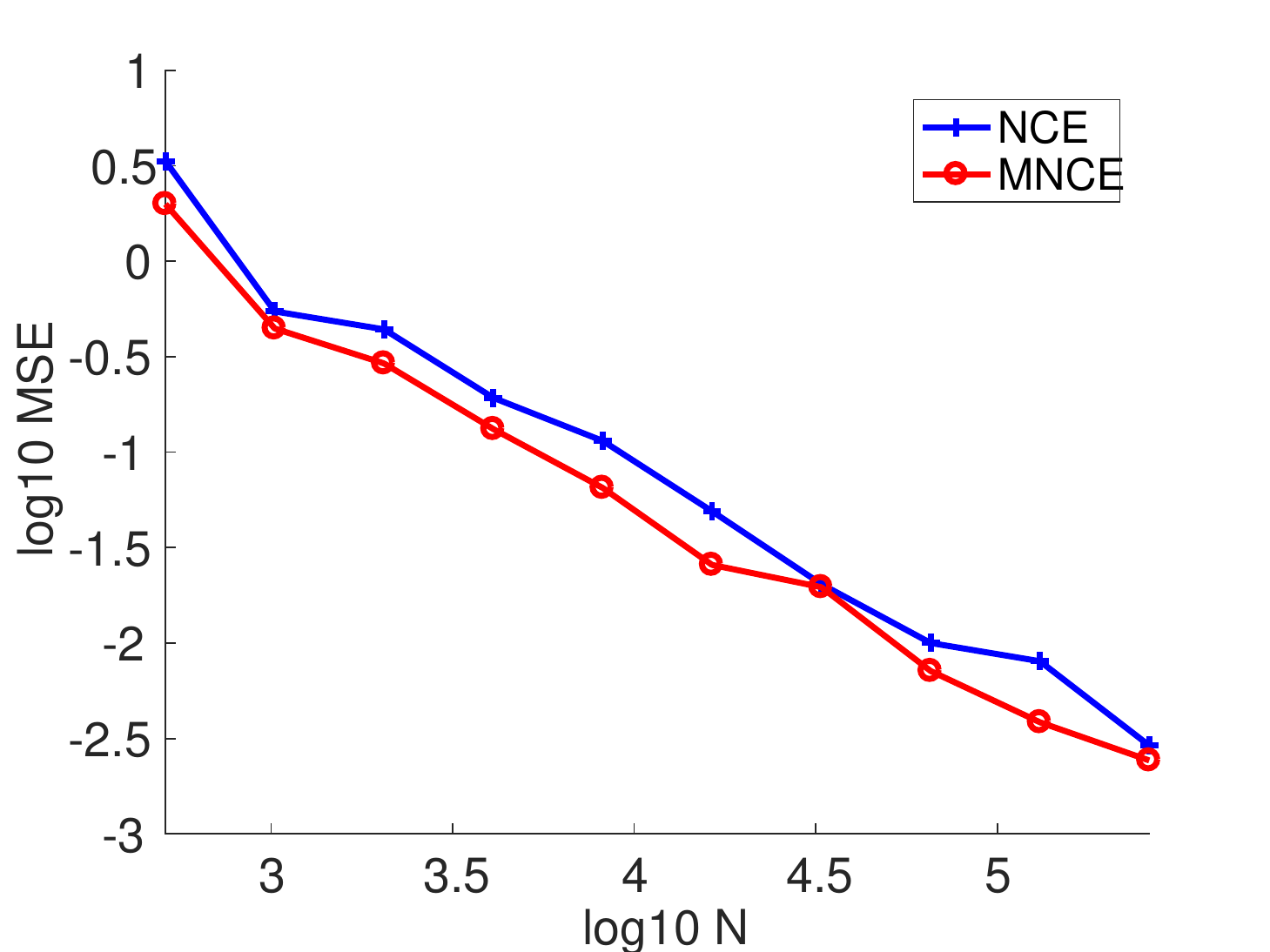}
	\end{center}
\end{minipage}
	\caption{Estimation errors for (a) $\theta$ and (b) $c$ in the Gaussian mixture distribution \eqref{gm_def}.}
	\label{fig_mse}
\end{figure}

\section{Application to real data}
In this section, we apply the proposed method to image clustering.
We use the training data of ``Dogs vs. Cats" competition at kaggle\footnote{https://www.kaggle.com/c/dogs-vs-cats} ($N=25000$), which consists of 12500 dog images and 12500 cat images.
As a pretrained network, we use inception-v3 \cite{Szegedy}, which extracts a $d=2048$ dimensional feature vector from image data.
This network was trained for ImageNet competition.
For noise samples, we use canine and feline images in the training data of inception-v3\footnote{From the 152-th category ``Chihuahua" to the 300-th category ``meerkat". We use only color images.} ($M=186125$, $L=149$).
We set the number of clusters to $K=2$.

We solved the optimization \eqref{NCEdef} in NCE by the nonlinear conjugate gradient method \cite{Rasmussen} with 10 random initial values of $(\theta,c)$. 
Among 10 converged solutions, we picked the one with the maximum value of objective function $\hat{J}_{{\rm MNCE}}$.

For comparison, we fitted the two-component Gaussian mixture model with diagonal covariance matrices to the feature vectors of $N$ images by using the MATLAB function \textit{fitgmdist}.
We also fitted the two-component Gaussian mixture model with isotropic covariance matrices by EM algorithm.
Although these models also provide clustering, it is heuristic and not probabilistically rigorous.

Figure \ref{fig_hist} shows the histogram of the posterior probability in the first cluster $p(z=1 \mid x; \hat{\theta},\hat{c})$.
Since the posterior takes values close to zero or one, almost all images are classified with high confidence.
Figure \ref{fig_hist2} shows the histogram of the logit score of the posterior probability in the first cluster $\log p(z=1 \mid x; \hat{\theta},\hat{c}) - \log (1-p(z=1 \mid x; \hat{\theta},\hat{c}))$.
Compared to the proposed method, the Gaussian mixture models assign extremely large or small logit scores 
and so it seems to fail to quantify the classification uncertainty properly.

\begin{figure}
\begin{minipage}{0.3\textwidth}
(a)
	\begin{center}
	\includegraphics[width=4cm]{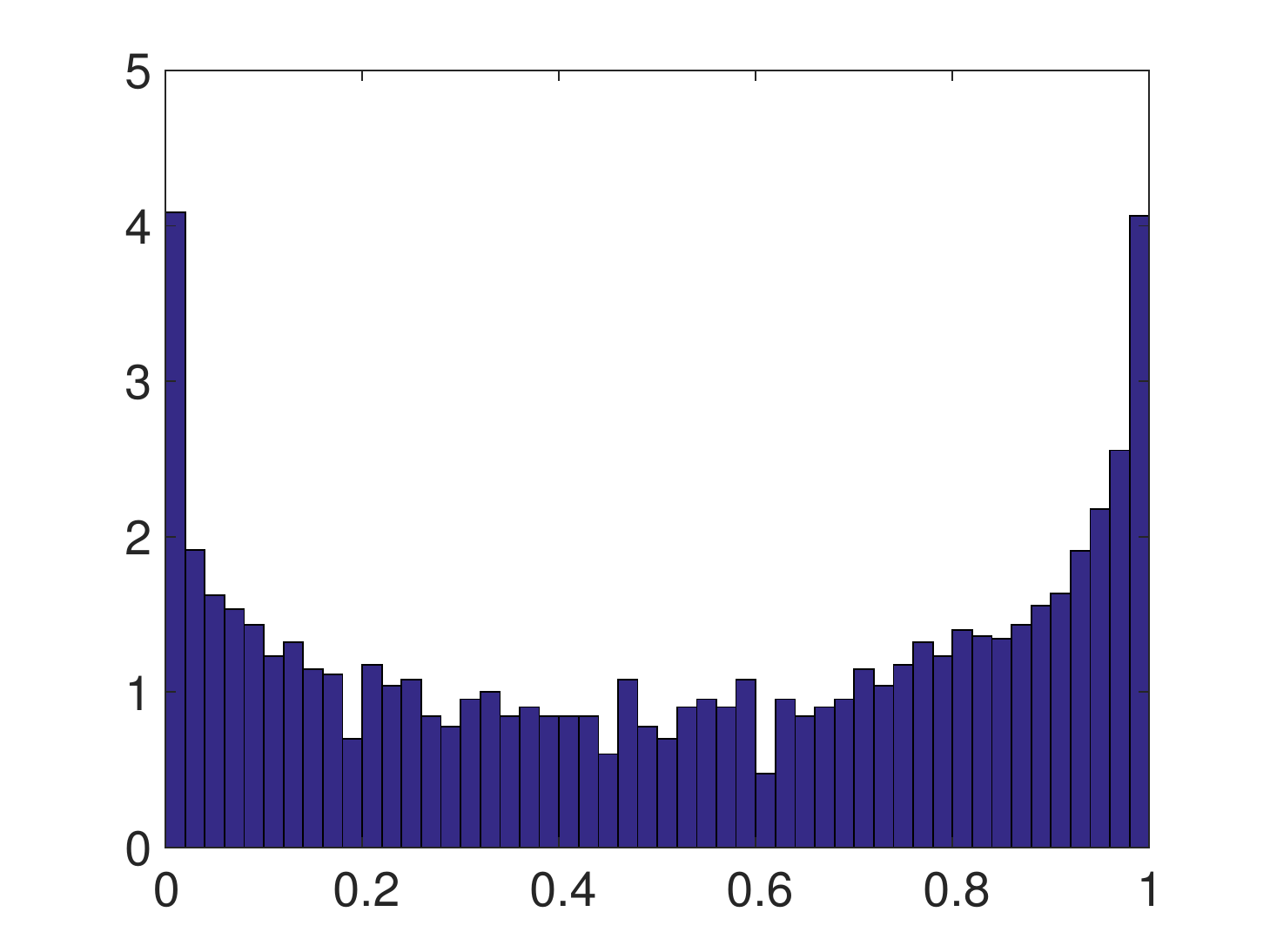}
	\end{center}
\end{minipage}
\begin{minipage}{0.3\textwidth}
(b)
	\begin{center}
	\includegraphics[width=4cm]{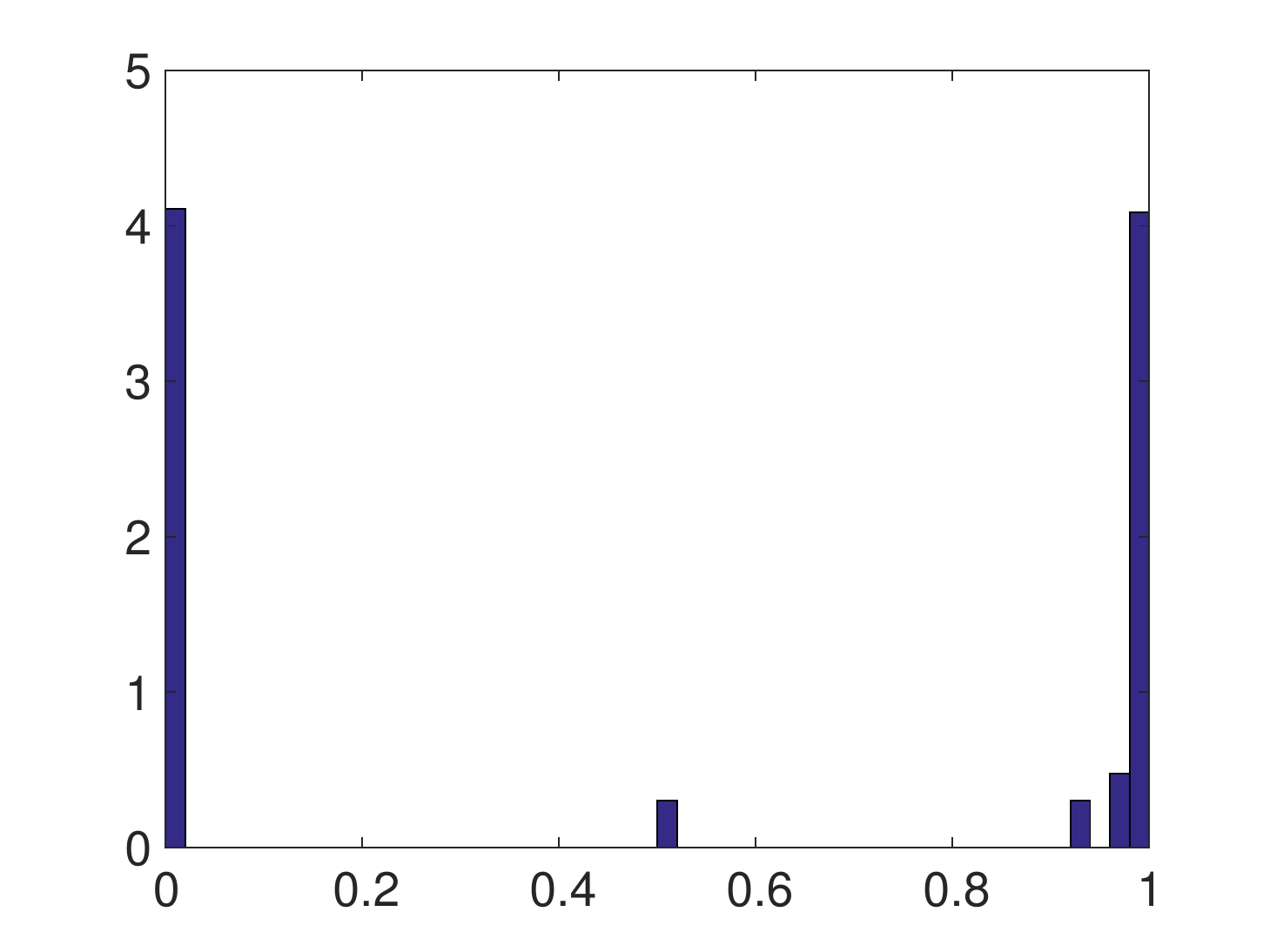}
	\end{center}
\end{minipage}
\begin{minipage}{0.3\textwidth}
(c)
	\begin{center}
	\includegraphics[width=4cm]{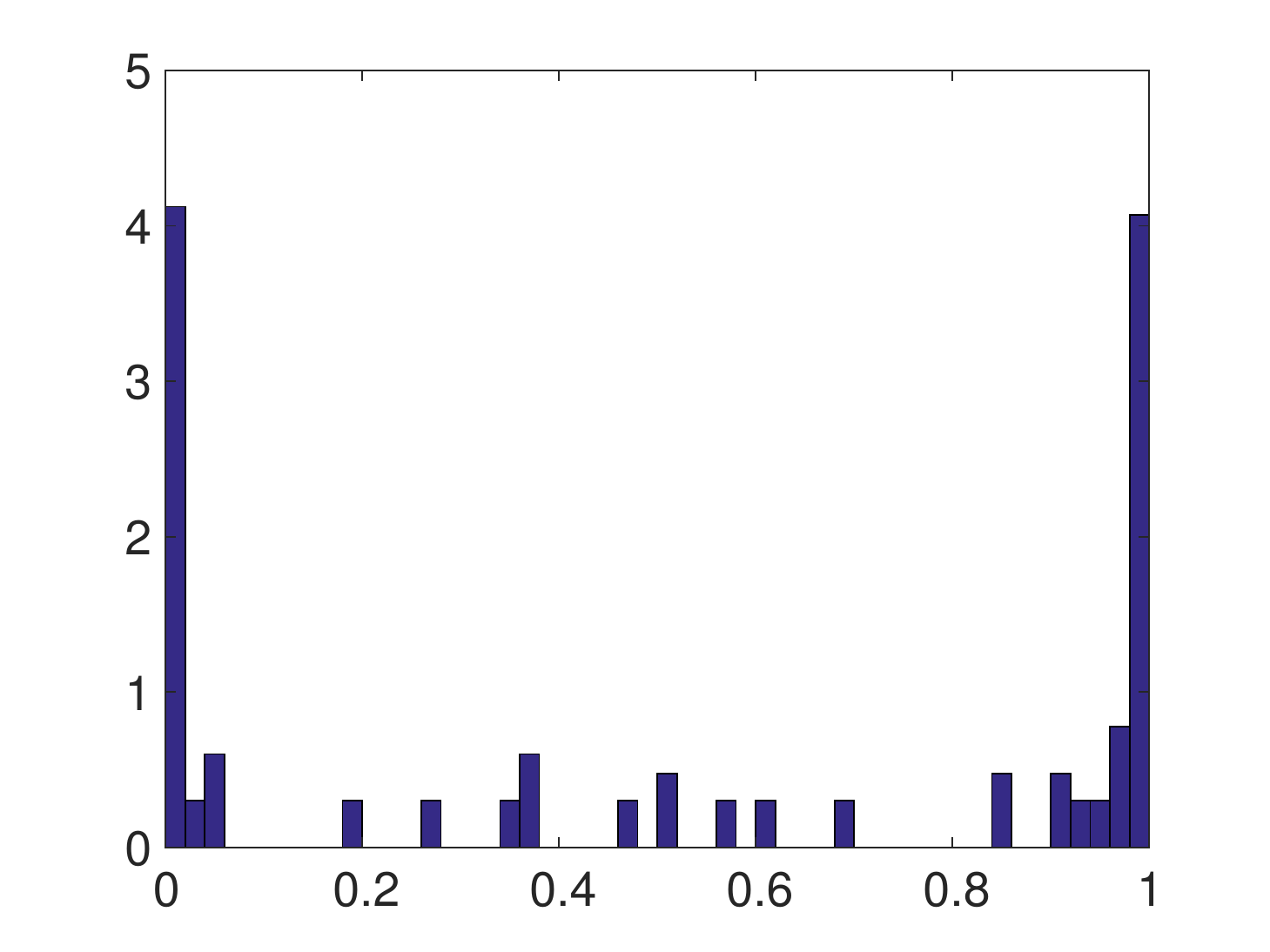}
	\end{center}
\end{minipage}
	\caption{Histogram of the posterior probability. The y-axis is in scale $\log_{10} (1+y)$, where $y$ is the frequency. (a) The proposed method. (b) Gaussian mixture model with diagonal covariance matrices. (c) Gaussian mixture model with isotropic covariance matrices.}
	\label{fig_hist}
\end{figure}

\begin{figure}
\begin{minipage}{0.3\textwidth}
(a)
	\begin{center}
	\includegraphics[width=4cm]{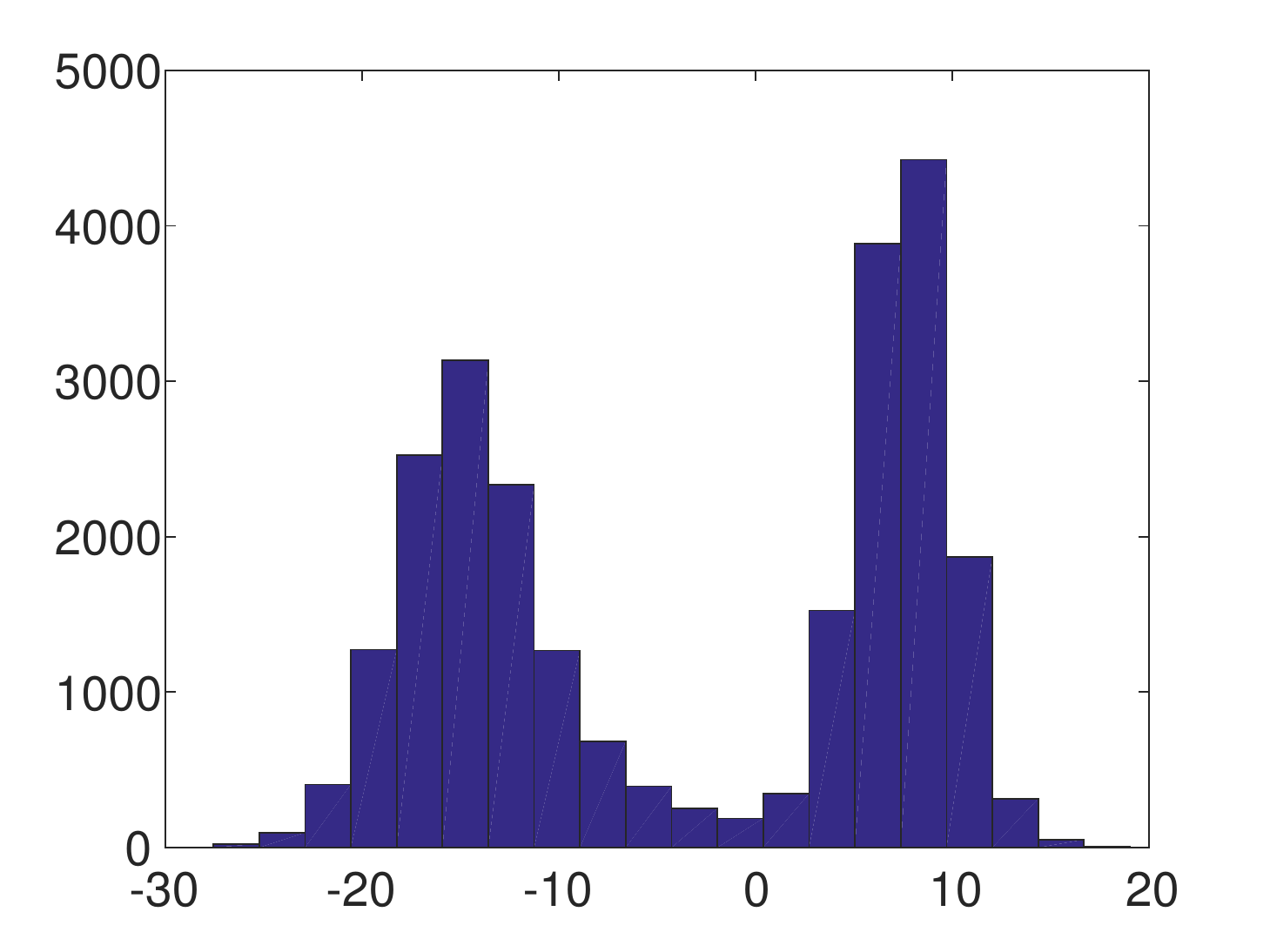}
	\end{center}
\end{minipage}
\begin{minipage}{0.3\textwidth}
(b)
	\begin{center}
	\includegraphics[width=4cm]{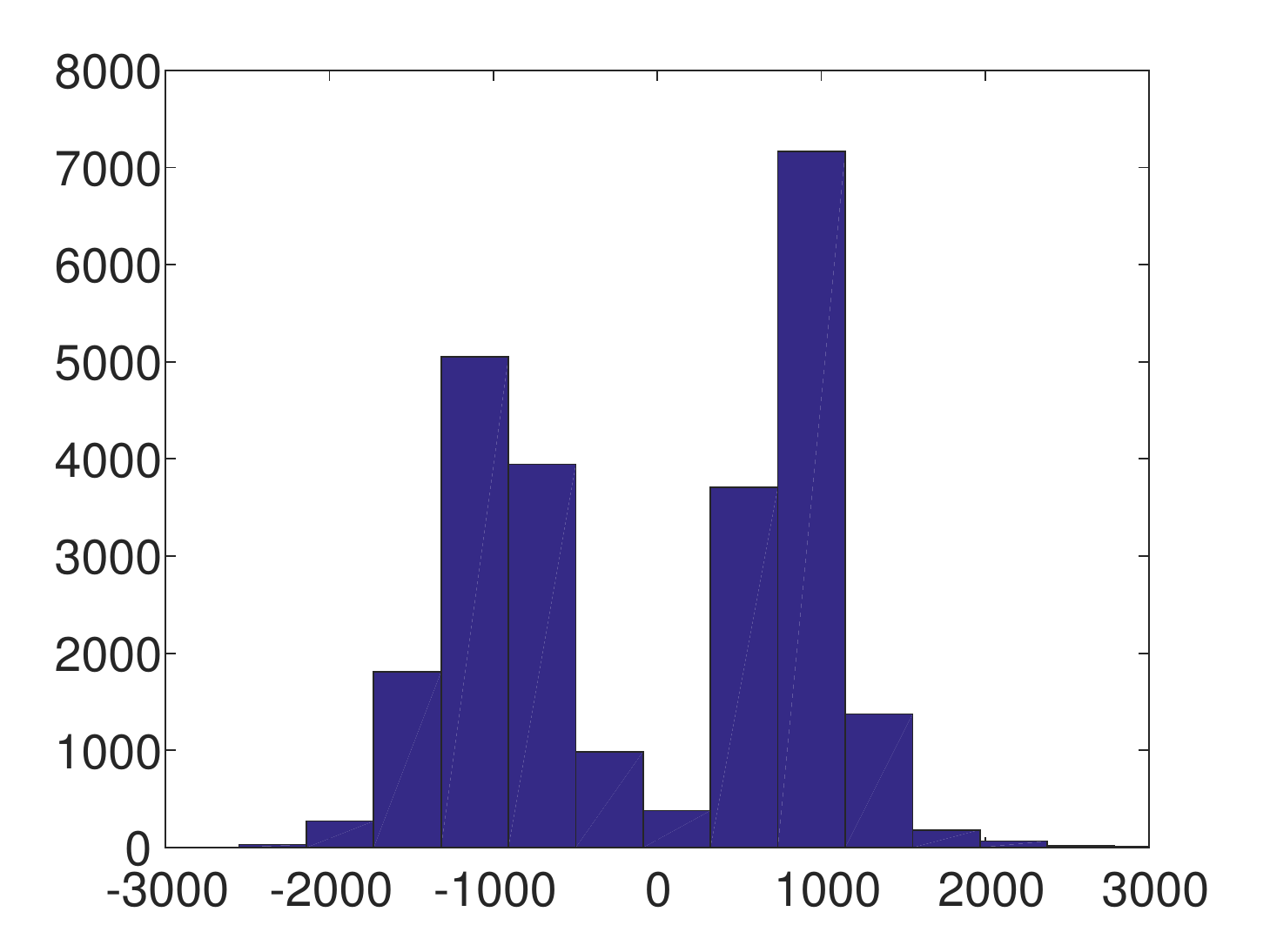}
	\end{center}
\end{minipage}
\begin{minipage}{0.3\textwidth}
(c)
	\begin{center}
	\includegraphics[width=4cm]{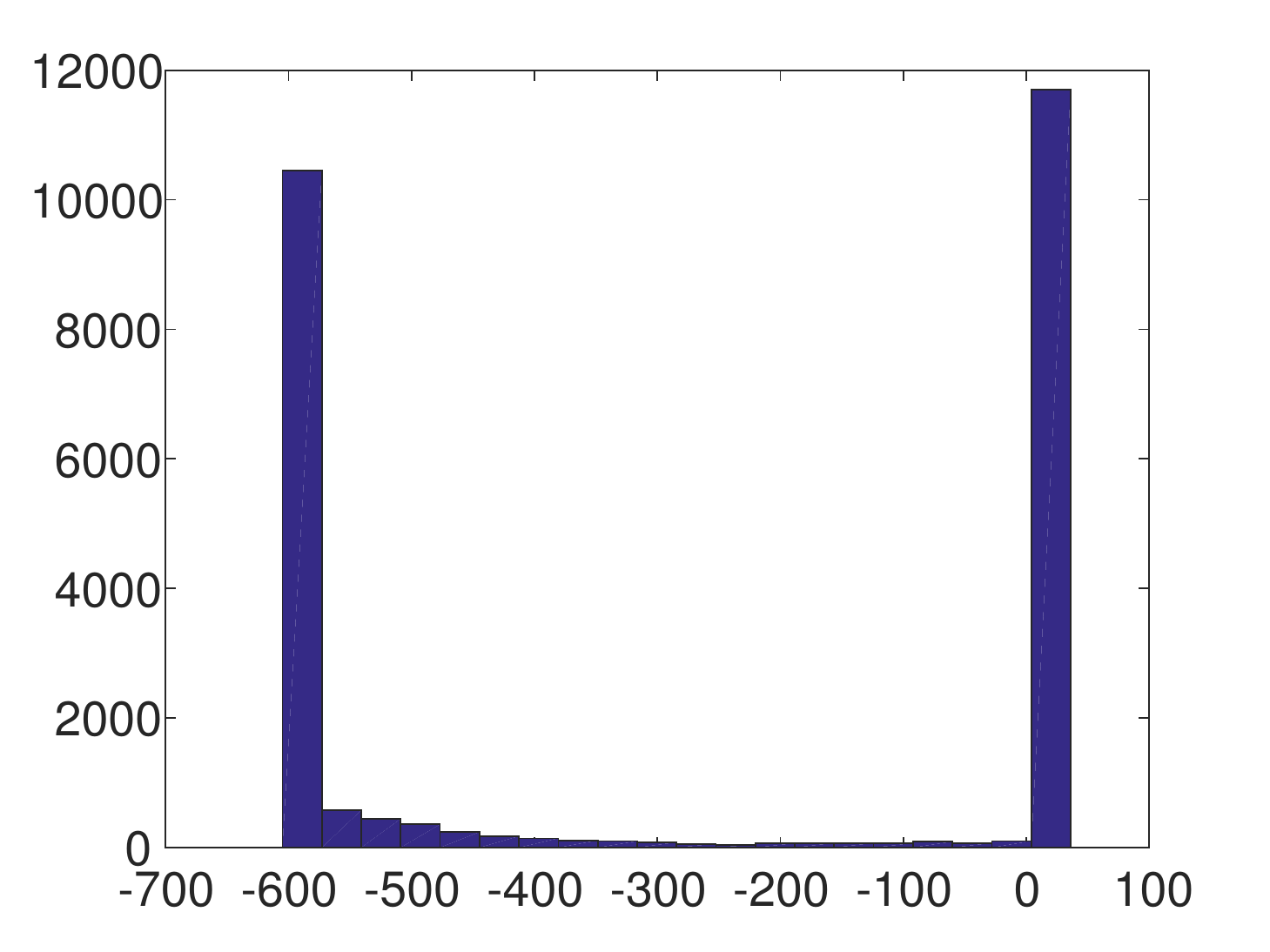}
	\end{center}
\end{minipage}
	\caption{Histogram of the logit score of posterior probability. Please note different horizontal ranges in the three plots. (a) The proposed method. (b) Gaussian mixture model with diagonal covariance matrices. (c) Gaussian mixture model with isotropic covariance matrices.}
	\label{fig_hist2}
\end{figure}

Table \ref{tab_dogcat} shows the clustering result.
Here, we classify an image $x$ into cluster $k$ if $p(z=k \mid x; \hat{\theta},\hat{c})>0.5$.
In all methods, the two clusters seem to separate dogs and cats well,
although the training of inception-v3 was done with more detailed categories like ``Scotch terrier" or ``snow leopard."
The proposed method has better classification accuracy compared to the Gaussian mixture models.

\begin{table}[htbp]
	\caption{Image clustering result. (a) The proposed method. (b) Gaussian mixture model with diagonal covariance matrices. (c) Gaussian mixture model with isotropic covariance matrices.}
	\label{tab_dogcat}
\begin{minipage}{0.3\textwidth}
	\begin{center}
	\begin{tabular}{|c|c|c|}
	\hline
	 a) & dog & cat \\ \hline
	cluster 1 & 12400 & 145 \\ \hline
	cluster 2 & 100 & 12355 \\ \hline
	\end{tabular}
	\end{center}
\end{minipage}
\begin{minipage}{0.3\textwidth}
	\begin{center}
	\begin{tabular}{|c|c|c|}
	\hline
	 b) & dog & cat \\ \hline
	cluster 1 & 12490 & 325 \\ \hline
	cluster 2 & 10 & 12175 \\ \hline
	\end{tabular}
	\end{center}
\end{minipage}
\begin{minipage}{0.3\textwidth}
	\begin{center}
	\begin{tabular}{|c|c|c|}
	\hline
	 c) & dog & cat \\ \hline
	cluster 1 & 12490 & 792 \\ \hline
	cluster 2 & 10 & 11708 \\ \hline
	\end{tabular}
	\end{center}
\end{minipage}
\end{table}

\section{Conclusion}
We extended noise contrastive estimation (NCE) to estimate a finite mixture of non-normalized models, and investigated NCE with multiple noise distributions.
Both theory and simulation results showed the validity of these extensions of NCE.

Based on the extended NCE, we proposed a method for clustering unlabeled data by using deep representation.
The clustering is attained by estimating a finite mixture of distributions in an exponential family; such a model is in fact implicitly assumed as a generative model in classification learning with neural networks.
For estimation, we use NCE in which the original training data of the neural network are used as noise.
Application to image clustering gave promising results on a simple task.

Although we considered image clustering here, the proposed method can be applied on other kinds of data, e.g., neuroimaging data.
By using the deep representation obtained by recently developed nonlinear ICA methods \cite{TCL,PCL}, 
clustering of brain states may be attained in an across-subject transfer setting, which is important in Brain Machine Interface (BMI) applications.


\clearpage

\begin{center}
\Large{Supplementary Material}
\end{center}

\section*{Consistency of the extended NCE (Section 3.2)}

Here, we consider the non-normalized mixture model
\begin{align}
	p (x \mid \theta,\pi) &= \sum_{k=1}^K \pi_k \cdot p (x \mid \theta_k), \label{mix_model_supp}
\end{align}
where the $K$ components belong to the same parametric family:
\begin{align}
	p (x \mid \theta_k) = \frac{1}{Z(\theta_k)} \widetilde{p}(x \mid \theta_k). \label{model_supp}
\end{align}
We also use the parametrization with $(\theta,c)$ defined by
\begin{align}
	p (x \mid \theta,c) = \sum_{k=1}^K p(x \mid \theta_k,c_k),
\end{align}
where
\begin{equation}
	\log p (x \mid \theta_k,c_k) = \log \widetilde{p} (x \mid \theta_k) + c_k.
\end{equation}
Two parametrizations are connected by the transformation $c_k = \log \pi_k-\log Z(\theta_k)$.

Suppose we have $N$ samples $x_1,\cdots,x_N$ from $p(x \mid \theta^{*},c^{*})$.
We consider the asymptotics where $N \to \infty$ with $\nu := M/N$ fixed, which is the same setting with Gutmann and Hyv\"arinen (2012).
Let
\begin{align}
	J_{{\rm NCE}} (\theta,c) =& \int p(x \mid \theta^{*},c^*) \log \frac{N p(x \mid \theta,c)}{N p(x \mid \theta,c) + M n(x)} {\rm d} x \nonumber \\
	&+ \nu \int n(y) \log \frac{M n(y)}{N p(y \mid \theta,c) + M n(y)} {\rm d} y.
\end{align}
Then, we obtain the following.

\begin{Lemma}
Assume the following. 
\begin{itemize}
\item[(a)] The set $\{ p(x \mid \theta) \mid \theta \in \Theta \}$ is linearly independent, where $\Theta$ is the parameter space of \eqref{model_supp}.

\item[(b)] The parameters $\theta^{*}_1,\cdots,\theta^{*}_K$ are all different.

\item[(c)] The parameters $\pi^{*}_1,\cdots,\pi^{*}_K$ are all nonzero.

\item[(d)] $n(x)$ is nonzero whenever $p(x \mid \theta^*,c^*)$ is nonzero.
\end{itemize}
Then, 
\begin{equation}
	{\rm arg} \max_{\theta,c} J_{{\rm NCE}} (\theta,c) = \{ (\theta^*_{\sigma(1)},\cdots,\theta^*_{\sigma(K)},c^*_{\sigma(1)},\cdots,c^*_{\sigma(K)}) \mid \sigma \in S_n \}, \label{lem_eq}
\end{equation}
where $S_n$ is the set of all permutations of $\{ 1,\cdots,K \}$.
\end{Lemma}
\begin{proof}
From Theorem 1 of Gutmann and Hyv\"arinen (2012) with assumption (d), 
\begin{equation}
	(\theta,c) \in {\rm arg} \max_{\theta,c} J_{{\rm NCE}} (\theta,c)
\end{equation}
if and only if
\begin{equation}
	p(x \mid \theta,c)=p(x \mid \theta^*,c^*),
\end{equation}
which is rewritten as
\begin{equation}
	\sum_{k=1}^K \pi_k p(x \mid \theta_k)=\sum_{k=1}^K \pi^*_k p(x \mid \theta_k^*).
\end{equation}
Then, from assumptions (a)-(c), it leads to 
\begin{equation}
	\{ \pi_1 p(x \mid \theta_1), \cdots, \pi_K p(x \mid \theta_K) \} = \{ \pi^*_1 p(x \mid \theta_1^*), \cdots, \pi^*_K p(x \mid \theta_K^*) \}.
\end{equation}
Therefore, there exists $\sigma \in S_n$ such that
\begin{equation}
	\pi_k p(x \mid \theta_k) = \pi^*_{\sigma(k)} p(x \mid \theta_{\sigma(k)}^*) \quad (k=1,\cdots,K),
\end{equation}
which is equivalent to
\begin{equation}
	\pi_k = \pi^*_{\sigma(k)}, \quad \theta_k = \theta_{\sigma(k)}^* \quad (k=1,\cdots,K)
\end{equation}
by using assumption (a).
Thus, we obtain \eqref{lem_eq}.
\end{proof}

In the above Lemma, assumption (a) holds for general exponential families including the Gaussian distribution.
Under this assumption, assumptions (b) and (c) mean that the true data-generating distribution has exactly $K$ components.
Assumption (d) is standard in noise contrastive estimation (Gutmann and Hyv\"arinen, 2012) and easily fulfilled by taking, for example, a Gaussian as the noise distribution.

Thus, the parameter in the mixture model \eqref{mix_model_supp} has indeterminacy with respect to the ordering of $K$ components. 
However, if we restrict the parameter space of \eqref{mix_model_supp} by putting order constraints on $\theta_1,\cdots,\theta_K$, 
the mixture model \eqref{mix_model_supp} becomes identifiable and so the true parameter value $(\theta^{*},\pi^{*})$ is defined uniquely.
For example, in Section 6, we sorted two Gaussian components by the mean.

After obtaining the identifiability of the mixture model \eqref{mix_model_supp} as above, the consistency of extended NCE is stated as follows.

\begin{Theorem}
Let $\xi=(\theta,c)$. Assume the following.
\begin{itemize}
\item $n(x)$ is nonzero whenever $p(x \mid \theta^*,c^*)$ is nonzero.

\item 
$\sup_{\xi} |N^{-1} \hat{J}_{{\rm NCE}}(\xi)-J_{{\rm NCE}}(\xi)| \overset{p}{\to} 0$.

\item The matrix $I=\int g(u) g(u)^{\top} P_{\nu} (u) p(u \mid \xi^{*}) {\rm d} u$ has full rank, where
\begin{align}
	g(u) = \left. \nabla \log_{\xi} p (u \mid \xi) \right|_{\xi=\xi^{*}}, \quad P_{\nu} = \frac{\nu n(u)}{p(u \mid \xi^{*}) + \nu n(u)}.
\end{align}
\end{itemize}
Then, $\hat{\xi}_{{\rm NCE}}$ in Section 3.2 converges in probability to $\xi^{*}$: $\hat{\xi}_{{\rm NCE}} \overset{p}{\to} \xi^{*}$.
\end{Theorem}
\begin{proof}
The proof is essentially the same with the original NCE.
See Theorem 2 of Gutmann and Hyv\"arinen (2012) for detail.
\end{proof}


\section*{Proof of Theorem 1 (Section 3.3)}
Since $n(y_t)$ does not depend on $\theta$ and $c$, we can rewrite (3) and (4) as
\begin{align}
	(\hat{\theta}_{{\rm NCE}},\hat{c}_{{\rm NCE}}) = {\rm arg} \max_{\theta,c} \widetilde{J}_{{\rm NCE}} (\theta,c),
\end{align}
where
\begin{align}
	\widetilde{J}_{{\rm NCE}} (\theta,c) =& \sum_{t=1}^N \log \frac{N p(x_t \mid \theta,c)}{N p(x_t \mid \theta,c)+M n(x_t)} + \sum_{t=1}^M \log \frac{1}{N p(y_t \mid \theta,c)+M n(y_t)}
\end{align}

Similarly, since $n_l(y^{(l)}_t)$ does not depend on $\theta$ and $c$, we can rewrite (12) and (13) as
\begin{align}
	(\hat{\theta}_{{\rm MNCE}},\hat{c}_{{\rm MNCE}}) = {\rm arg} \max_{\theta,c} \widetilde{J}_{{\rm MNCE}} (\theta,c),
\end{align}
where
\begin{align}
	\widetilde{J}_{{\rm MNCE}} (\theta,c) =& \sum_{t=1}^N \log \frac{N p(x_t \mid \theta,c)}{N p(x_t \mid \theta,c)+M_1 n_1(x_t)+\cdots+M_L n_L(x_t)} \nonumber \\
	&+ \sum_{l=1}^L \sum_{t=1}^{M_l} \log \frac{1}{N p(y^{(l)}_t \mid \theta,c)+M_1 n_1(y^{(l)}_t)+\cdots+M_L n_L(y^{(l)}_t)}.
\end{align}

Now, from (14), we obtain $\widetilde{J}_{{\rm NCE}} (\theta,c) = \widetilde{J}_{{\rm MNCE}} (\theta,c)$. 
Therefore,
\begin{align}
	(\hat{\theta}_{{\rm MNCE}},\hat{c}_{{\rm MNCE}}) = (\hat{\theta}_{{\rm NCE}},\hat{c}_{{\rm NCE}}).
\end{align}


\clearpage

\begin{thebibliography}{99}
\bibitem{Besag}
\textsc{Besag, J.} (1974).
\newblock{Spatial interaction and the statistical analysis of lattice systems}.
\newblock \textit{Journal of the Royal Statistical Society B}, \textbf{36}, 192--236.


\bibitem{Dai}
\textsc{Dai, J.}, \textsc{Lu, Y.} \& \textsc{Wu, Y. N.} (2015).
\newblock{Generative modeling of convolutional neural networks}.
\newblock In \textit{Proceedings of the 3rd International Conference on Learning Representations (ICLR)}.

\bibitem{GAN}
\textsc{Goodfellow, I.}, \textsc{Pouget-Abadie, J.}, \textsc{Mirza, M.}, \textsc{Xu, B.}, \textsc{Warde-Farley, D.}, \textsc{Ozair, S.}, \textsc{Courville, A.} \& \textsc{Bengio, J.} (2014).
\newblock{Generative Adversarial Nets}.
\newblock In \textit{Advances in Neural Information Processing Systems 27}.


\bibitem{Gutmann}
\textsc{Gutmann, M. U.} \& \textsc{Hyv\"arinen, A.} (2010).
\newblock{Noise-contrastive estimation: A new estimation principle for non-normalized statistical models}.
\newblock In \textit{Proceedings of the 13th International Workshop on Artificial Intelligence and Statistics (AISTATS)}.

\bibitem{Gutmann12}
\textsc{Gutmann, M. U.} \& \textsc{Hyv\"arinen, A.} (2012).
\newblock{Noise-contrastive estimation of non-normalized statistical models, with applications to natural image statistics}.
\newblock \textit{Journal of Machine Learning Research}, \textbf{13}, 307--361.

\bibitem{Hinton}
\textsc{Hinton, G. E.} (2002).
\newblock{Training products of experts by minimizing contrastive divergence}.
\newblock \textit{Neural Computation}, \textbf{14}, 1771--1800.

\bibitem{SM}
\textsc{Hyv\"arinen, A.} (2005).
\newblock{Estimation of non-normalized statistical models by score matching}.
\newblock \textit{Journal of Machine Learning Research}, \textbf{6}, 695--709.

\bibitem{TCL}
\textsc{Hyv\"arinen, A.} \& \textsc{Morioka, H.} (2016).
\newblock{Unsupervised feature extraction by time-contrastive learning and nonlinear ICA}.
\newblock In \textit{Advances in Neural Information Processing Systems 29}.

\bibitem{PCL}
\textsc{Hyv\"arinen, A.} \& \textsc{Morioka, H.} (2017).
\newblock{Nonlinear ICA of temporally dependent stationary sources}.
\newblock In \textit{Proceedings of the 20th International Workshop on Artificial Intelligence and Statistics (AISTATS)}.


\bibitem{Krizhevsky}
\textsc{Krizhevsky, A.}, \textsc{Sutskever, I.} \& \textsc{Hinton, G. E.} (2012).
\newblock{ImageNet classification with deep convolutional neural networks}.
\newblock In \textit{Advances in Neural Information Processing Systems 25}.


\bibitem{Li}
\textsc{Li, S. Z.} (2001).
\newblock \textit{Markov Random Field Modeling in Image Analysis}.
\newblock Springer.


\bibitem{Nair}
\textsc{Nair, V.} \& \textsc{Hinton, G.} (2008).
\newblock{Implicit mixtures of restricted Boltzmann machines}.
\newblock In \textit{Advances in Neural Information Processing Systems 21}.


\bibitem{Rasmussen}
\textsc{Rasmussen, C. E.} (2006).
\newblock{Conjugate gradient algorithm}.
\newblock Matlab code version 2006-09-08. \url{http://learning.eng.cam.ac.uk/carl/code/minimize/minimize.m}

\bibitem{Szegedy}
\textsc{Szegedy, C.}, \textsc{Vanhoucke, V.}, \textsc{Ioffe, S.}, \textsc{Shlens, J.}, \& \textsc{Wojna, Z.} (2015).
\newblock{Rethinking the inception architecture for computer vision}.
\newblock arXiv:1512.00567.


\bibitem{Teh}
\textsc{Teh, Y.}, \textsc{Welling, M.}, \textsc{Osindero, S.} \& \textsc{Hinton, G. E.} (2004).
\newblock{Energy-based models for sparse overcomplete representations}.
\newblock \textit{Journal of Machine Learning Research}, \textbf{4}, 1235--1260.

\bibitem{Xie}
\textsc{Xie, J.}, \textsc{Lu, Y.}, \textsc{Zhu, S. C.} \& \textsc{Wu, Y. N.} (2016).
\newblock{A theory of generative ConvNet}.
\newblock In \textit{Proceedings of the 33th Annual International Conference on Machine Learning (ICML)}.
\end{thebibliography}
\end{document}